\documentclass{IEEEtran}
\usepackage{amsmath}
\usepackage{amssymb}
\usepackage{amsthm}
\usepackage{algorithm,algorithmicx,algpseudocode}
\usepackage{graphicx}
\usepackage{epstopdf}
\usepackage{xcolor}
\usepackage{cite}
\usepackage[pdftex,colorlinks,linkcolor=blue,anchorcolor=blue,citecolor=blue]{hyperref}
\usepackage{diagbox}
\usepackage{makecell}
\usepackage{mycommand}
\usepackage{booktabs}
\usepackage[capitalize,noabbrev]{cleveref}
\crefname{assumption}{Assumption}{Assumptions}

\newtheorem{lemma}{Lemma}
\newtheorem{theorem}{Theorem}
\newtheorem{assumption}{Assumption}
\newtheorem{corollary}{Corollary}
\newtheorem{definition}{Definition}

\allowdisplaybreaks[4]

\usepackage{framed}
\colorlet{shadecolor}{blue!20}

\title{A Variance-Reduced Cubic-Regularized Newton for Policy Optimization}
\author{Cheng Sun, Zhen Zhang, Shaofu Yang
\thanks{C. Sun, Z. Zhang, and S. Yang are with the School of Computer Science and Engineering, Southeast University, Nanjing, China 
(email:\{cheng\_sun, zhenz, sfyang\}@seu.edu.cn).}
\thanks{\it Corresponding author: Shaofu Yang (sfyang@seu.edu.cn).}}

\begin{document}
\maketitle

\begin{abstract}
    In this paper, we study a second-order approach to policy optimization in reinforcement learning. Existing second-order methods often suffer from suboptimal sample complexity or rely on unrealistic assumptions about importance sampling. To overcome these limitations, we propose VR-CR-PN, a \emph{variance-reduced cubic-regularized policy Newton} algorithm. To the best of our knowledge, this is the first algorithm that integrates Hessian-aided variance reduction with second-order policy optimization, effectively addressing the distribution shift problem and achieving best-known sample complexity under general nonconvex conditions but without the need for importance sampling. We theoretically establish that VR-CR-PN achieves a sample complexity of $\tilde{\mathcal{O}}(\epsilon^{-3})$ to reach an $\epsilon$-second-order stationary point, significantly improving upon the previous best result of $\tilde{\mathcal{O}}(\epsilon^{-3.5})$ under comparable assumptions. As an additional contribution, we introduce a novel Hessian estimator for the expected return function, which admits a uniform upper bound independent of the horizon length $H$, allowing the algorithm to achieve horizon-independent sample complexity.
	\end{abstract}


\section{Introduction}
	\label{Introduction}
        Policy optimization is a fundamental approach in reinforcement learning (RL), aiming to learn a parameterized policy that maximizes expected cumulative reward, and it typically relies on first-order methods due to their simplicity and low per-iteration cost \cite{sutton1999policy,schulman2015trust,schulman2017proximal,kakade2001natural}.
        However, their convergence can be slow and they often get trapped in suboptimal first-order stationary points. Second-order methods, in contrast, offer the potential to accelerate convergence and improve final performance by driving the policy toward second-order stationary points, which better characterize local optimality in non-convex settings.
        
        While recent advances such as Cubic-Regularized Policy Newton (CR-PN) methods have demonstrated the feasibility of second-order optimization in RL\cite{pmlr-v238-maniyar24a}, these methods still suffer from suboptimal sample complexity, primarily due to the lack of variance reduction mechanisms. As sample efficiency is a central concern in RL—where data collection is often costly—reducing sample complexity is not only theoretically desirable but also practically critical.
        Variance reduction mechanisms have been widely applied in optimization to improve sample efficiency. However, unlike supervised learning, policy-dependent data distributions in RL introduce distribution shift when estimating gradients or Hessians across different policy iterates. This makes classical variance reduction techniques (e.g., SVRG\cite{johnson2013accelerating}, PAGE\cite{li2021page}) prone to bias unless importance sampling is employed. Yet, importance sampling itself requires restrictive assumptions that are hard to justify or verify in practice.
        
        Recent work by Masiha et al. \cite{masiha2022stochastic} incorporates variance reduction into second-order policy optimization by leveraging the Polyak–Łojasiewicz (PL) condition. Their approach achieves stronger sample complexity bounds, but at the expense of stricter assumptions, such as the weak gradient dominance property and the boundedness assumption on importance sampling. 

        A natural question is, can variance reduction be integrated with second-order policy gradient methods under milder assumptions?
        To address this question, we propose VR-CR-PN, a variance-reduced cubic-regularized Newton algorithm for policy optimization. 
        Unlike previous approaches, our method bypasses the need for importance sampling by introducing a novel technique based on higher-order Taylor expansion, which estimates gradient differences via Hessian-vector products. This allows for variance reduction across policy iterates without incurring bias, even in the presence of distributional shift.
        The proposed algorithm overcome the limitations of existing second-order policy optimization methods, achieves convergence to second-order stationary points with an improved sample complexity of $\tilde{O}(\epsilon^{-3})$, matching the theoretical lower bound under standard assumptions and outperforming the state-of-the-art complexity of $\tilde{O}(\epsilon^{-3.5})$.

        In addition to the contributions above, we address a critical limitation of existing Hessian estimators in long-horizon policy optimization. By analyzing the estimator structure, we identify and eliminate redundant terms that inflate the norm bound, and design a refined estimator that enables scalable variance reduction even under extended truncation lengths. Detailed construction and analysis are provided in \cref{sec:new_hessian_estimator}.

        Our contributions include:
        \begin{itemize}
            \item A new algorithm that converges to a second-order stationary point with sample complexity improved to $\tilde{O}(\epsilon^{-3})$, \textbf{outperforming the state-of-the-art complexity} of $\tilde{O}(\epsilon^{-3.5})$, and matching the theoretical lower bound under standard assumptions(see \cref{tab:compare-sample-complexity}). To our knowledge, this is the first second-order policy optimization method in reinforcement learning that integrates variance reduction without relying on importance sampling.
            \item A new Hessian estimator with truncation-independent bound. We propose a refined estimator that removes unnecessary zero-mean terms in the conventional Hessian estimation, resulting in a significantly tighter upper norm bound. Notably, this bound is independent of the truncation horizon $H$, allowing the sample size $m$ to be chosen independently of $H$ and thereby enabling sample-efficient learning even under long-horizon scenarios. We also establish a lower bound on the required $H$ for this estimator.
            \item Empirical validation showcasing practical effectiveness. Empirical experiments on the CartPole-v1 benchmark demonstrate that VR-CR-PN achieves consistently higher return and lower variance compared to CR-PN under identical sample complexity. The results validate the effectiveness of our variance-reduced estimator in accelerating convergence and stabilizing training, confirming the practical advantages predicted by our theoretical analysis.
        \end{itemize}
    
\section{Related Work}
    \paragraph{Variance Reduction.}
    Variance reduction techniques play an important role in improving the sample efficiency of stochastic optimization\cite{paczolay2024sample}. Classical methods such as SVRG~\cite{johnson2013accelerating}, SAGA~\cite{defazio2014saga}, and STORM~\cite{cutkosky2019momentum} have been widely applied in supervised learning, where gradient estimators constructed from mini-batches and periodically updated reference points are effective under i.i.d.~data assumptions.

In reinforcement learning, however, the data distribution evolves with the policy, violating these assumptions. Direct application of such methods is challenging due to distribution shift across iterates. Existing approaches in RL, such as SVRPG~\cite{papini2018stochastic}, SRVR-PG~\cite{xu2019sample}, or PAGE-PG~\cite{gargiani2022page}, introduce correction terms to preserve unbiasedness and reduce variance. However, these methods often rely on importance sampling, which in turn necessitates stronger assumptions to mitigate the induced bias from reused samples.

Regarding this issue, our method builds on the estimator proposed by NPG-HM~\cite{feng2024global}, and integrates it into a second-order optimization framework without importance sampling, while maintaining convergence guarantees under standard smoothness assumptions.
\paragraph{Second-Order Policy Optimization.}
There has been increasing interest in applying second-order optimization to reinforcement learning, motivated by its potential to accelerate convergence and escape saddle points. Yang et al.~\cite{yang2021sample} analyze the sample complexity of vanilla policy gradient for finding second-order stationary points (SOSP), establishing a $\widetilde{\mathcal{O}}(\epsilon^{-9/2})$ bound without relying on strong convexity or PL conditions. Wang et al.~\cite{wang2022stochastic} propose a stochastic cubic-regularized policy gradient (SCR-PG) method, which improves the rate to $\widetilde{\mathcal{O}}(\epsilon^{-7/2})$ by leveraging cubic regularization and Hessian-vector products.

More recently, Maniyar et al.~\cite{pmlr-v238-maniyar24a} revisit cubic-regularized Newton methods in RL, proposed the algorithm CR-PN and provide a more complete and streamlined analysis. While both works achieve comparable complexity results, CR-PN presents a simplified algorithmic variant and introduces a distinct proof strategy that simplifies the analysis under verifiable assumptions. Our work builds upon this foundation by incorporating variance reduction into the cubic-regularized Newton framework, leading to the first variance-reduced second-order method for policy optimization with \textbf{improved theoretical and empirical sample efficiency}. 
\paragraph{Sample Complexity Analysis.}
Under standard smooth nonconvex assumptions, traditional first-order methods achieve a time complexity of $\tilde{\mathcal{O}}(\epsilon^{-2})$ and a sample complexity of $\tilde{\mathcal{O}}(\epsilon^{-4})$. Variance reduction techniques do not alter the time complexity but reduce the sample complexity to $\tilde{\mathcal{O}}(\epsilon^{-3})$. In contrast, classical second-order methods achieve both improved time complexity of $\tilde{\mathcal{O}}(\epsilon^{-3/2})$ and better sample complexity of $\tilde{\mathcal{O}}(\epsilon^{-7/2})$ compared to traditional first-order methods. Applying variance reduction to these methods similarly preserves the time complexity while reducing the sample complexity to $\tilde{\mathcal{O}}(\epsilon^{-3})$. See \cref{tab:compare-sample-complexity} for a detailed comparison.
Our method achieves a sample complexity of $\widetilde{\mathcal{O}}(\epsilon^{-3})$ for estimating gradients and $\widetilde{\mathcal{O}}(\epsilon^{-5/2})$ for estimating Hessians, under standard smoothness assumptions. To the best of our knowledge, these are the strongest known complexity bounds among second-order policy optimization methods that do not rely on restrictive structural conditions.

In summary, our work bridges two previously separate lines of research—cubic-regularized second-order optimization and variance reduction—in the context of reinforcement learning, and provides both improved theoretical guarantees and practical implications.

\begin{table}[tp!]
\centering
\caption{Comparison between our algorithm and existing algorithms. \textbf{IC:} Iteration complexity; \textbf{SC:} Sample complexity; \textbf{SOSP:} Converging to the second-order stationary point; \textbf{IS:} Requiring importance sampling.}
\label{tab:compare-sample-complexity}
\renewcommand{\arraystretch}{1.2}
\begin{tabular}{lcccc}
\toprule
\textbf{Algorithm} &
\makecell[c]{\textbf{IC}} &
\makecell[c]{\textbf{SC}} &
\makecell[c]{\textbf{SOSP}} &
\makecell[c]{\textbf{IS}} \\
\midrule
PG & $\tilde{\mathcal{O}}(\epsilon^{-2})$ & $\tilde{\mathcal{O}}(\epsilon^{-4})$ & \textcolor{red}{$\times$} & \textcolor{green!50!black}{$\checkmark$} \\
SVRPG & $\tilde{\mathcal{O}}(\epsilon^{-2})$ & $\tilde{\mathcal{O}}(\epsilon^{-10/3})$ & \textcolor{red}{$\times$} & \textcolor{red}{$\times$} \\
PAGE-PG & $\tilde{\mathcal{O}}(\epsilon^{-2})$ & $\tilde{\mathcal{O}}(\epsilon^{-3})$ & \textcolor{red}{$\times$} & \textcolor{red}{$\times$} \\
NPG-HM & $\tilde{\mathcal{O}}(\epsilon^{-2})$ & $\tilde{\mathcal{O}}(\epsilon^{-3})$ & \textcolor{red}{$\times$} & \textcolor{green!50!black}{$\checkmark$} \\
CR-PN & $\tilde{\mathcal{O}}(\epsilon^{-3/2})$ & $\tilde{\mathcal{O}}(\epsilon^{-7/2})$ & \textcolor{green!50!black}{$\checkmark$} & \textcolor{green!50!black}{$\checkmark$} \\
\textbf{VR-CR-PN} & $\tilde{\mathcal{O}}(\epsilon^{-3/2})$ & $\tilde{\mathcal{O}}(\epsilon^{-3})$ & \textcolor{green!50!black}{$\checkmark$} & \textcolor{green!50!black}{$\checkmark$} \\
\bottomrule
\end{tabular}
\end{table}

\paragraph{Notation}
\begin{itemize}
    \item Given a function $f:\RR^d \to \RR$, we use $\nabla^2 f \in \RR^{d\times d}$ and $\nabla^3 f \in \RR^{d\times d\times d}$ to represent the second-order and third-order derivative of $f$, respectively. 
    \item Given two tensors $A \in \RR^{ {d_1} \times \cdots \times d_{m}}$ and $B \in \RR^{p_{1} \times \cdots \times p_{n}}$, we use $\otimes$ to denote the tensor product as below: $A\otimes B \in \RR^{{d_1} \times \cdots \times d_{m}\times p_{1} \times \cdots \times p_{n} }$ with
    $$(A \otimes B)_{i_1, \dots, i_m, j_1, \dots, j_n} = A_{i_1, \dots, i_m} \cdot B_{j_1, \dots, j_n}.$$
    For example, suppose 
    $A= 
    \begin{pmatrix}
        1 & 2 \\ 3 & 4
    \end{pmatrix} \in \RR^{2\times 2}
    , B^{\tr}= [0,1,2]\in \RR^{1\times 3}$. 
    Then, $A \otimes B \in \mathbb{R}^{2 \times 2 \times 3}$ and 
    \[
    (A \otimes B)_{1,:,:} = 
    \begin{bmatrix}
    0 & 1 & 2 \\
    0 & 2 & 4 
    \end{bmatrix}
    \in \mathbb{R}^{2 \times 3}.
    \]
\end{itemize}
\section{Policy Gradient Framework}
        \label{Policy gradient framework}
        A Markov decision process (MDP) can be formally defined as a mathematical framework characterized by six fundamental components: $(\sS, \sA, p, r, \gamma, \rho)$. Here, $\sS$ corresponds to the set of all possible states in the system, while $\sA$ represents the collection of available actions. The transition dynamics are governed by the probability function $p(s'|s, a)$, which determines the likelihood of moving from state $s$ to state $s'$ when action $a$ is executed by the decision-making agent. The immediate consequence of taking action a in state $s$ is quantified through the reward function $r(s, a)$. Two additional parameters complete the specification: $\gamma\in[0, 1)$ serves as a discount factor that weights future rewards, and $\rho$ describes the initial state distribution. Within this framework, actions are selected based on a stochastic policy $\pi$, where $\pi(\cdot\mid s)$ denotes the probability distribution over actions in state $s$. For analytical convenience, we maintain the standard assumption that all actions are admissible in every state.

        Given a random initial state $s_0 \sim \rho$ and a policy $\pi$, we can generate a random trajectory: 
        $$\tau := (s_0,a_0,s_1,a_1,\cdots,s_{T-1},a_{T-1},s_T).$$
        Here, $T$ refers to the length of the trajectory, $s_T$ is called the terminal state. For different trajectory, $T$ can be different. Clearly, given $\rho$ and $\pi$, the probability of trajectory $\tau$ is 
        \begin{align}
           \Pr(\tau \mid \rho, \pi) := \rho(s_0)\prod_{h=0}^{T-1} \pi(a_h|s_h) p(s_{h+1}|s_h, a_h). 
        \end{align}
In addition, we denote $\tau^{[t]}:=\{s_0,a_0,\cdots,s_t\}$ as a sub-trajectory from $\tau$. Then, we have
        \begin{align}
           \Pr(\tau^{[t]} \mid \rho, \pi) := \rho(s_0)\prod_{h=0}^{t-1} \pi(a_h|s_h) p(s_{h+1}|s_h, a_h). 
        \end{align}

        To evaluate the total reward accumulated along a trajectory, we adopt the notion of return, defined as the cumulative discounted sum of rewards:
        \begin{align}
            G(\tau) := \sum_{h=0}^{{T}-1} \gamma^{h} r(s_h, a_h),
        \end{align}
        where $\gamma\in [0,1) $ is the discount factor. 
        Accordingly, we define our objective function $J(\pi)$ as:
        \begin{align}\label{J_pi}
        J(\pi)=\EE_{\tau\sim\Pr(\cdot\mid\rho,\pi)}G(\tau).
        \end{align}
        The central objective is to determine an optimal policy $\pi^*$ that maximizes the expected discounted return~(\ref{J_pi}) in a given environment.

        Since the policy space may be high-dimensional or even infinite, a standard approach is to assume that policies are governed by a finite set of parameters $\theta$. For brevity, we adopt the shorthand $\Pr(\tau|\theta)$ for $\Pr(\tau|\rho,\pi)$, and $J(\theta)$ for $J\left(\pi(\theta)\right)$.
        
        The conventional formulation of Markov decision problems is about maximizing the
        discounted total reward. In this paper, we adopt a minimization formulation in order to
        better align with conventions in the optimization literature. To this end, we regard each
        $r(s, a) \in [-R_{\max}, R_{\max}]$ as a value measuring cost rather than reward. Given a reward matrix $R$,
        we can reset $r(s, a) \leftarrow -r(s, a)$ for all $(s, a) \in S \times A$ to turn it into a cost matrix.        
        Our aim is to find a parameter to minimize the expected discounted cumulative cost.
        \begin{align}
         \theta^* \in \arg\min_{\theta \in \mathbb{R}^d} J(\theta)\label{J_vtheta}   
        \end{align}
        In this work, we introduce a novel second-order policy gradient algorithm to address Problem (\ref{J_vtheta}). The development of this approach requires explicit expressions for both the gradient and the Hessian matrix of the objective function. Building upon existing theoretical frameworks, we derive a new representation for the Hessian matrix, the details of which are presented in the following sections.

\section{Formulation of Our Algorithm}

        \subsection{A New Hessian Estimator}
        \label{sec:new_hessian_estimator}
        To design second-order algorithms, an essential task is to estimate the Hessian matrix.
        First, we define a variable for describing the policy gradient or Hessian in brevity. Given a trajectory $\tau$, define
        \begin{align}
            X(t;\tau) = \sum_{k=0}^{t}\log\pi(a_k|s_k), \quad 0\leq t \leq T-1.
        \end{align}
        Based on the policy gradient theorem, the gradient of $J(\theta)$ can be expressed as
        \begin{align}
        \label{eq:gradient-expression}
            \nabla J(\theta)&=\EE\sum_{k=0}^{\infty}
            \gamma^k r_k \nabla X(k;\tau).
        \end{align}
        Having established the gradient estimator, we now turn to the second-order properties of the objective function. We propose a new Hessian matrix expression that eliminates the dependence on the full trajectory length $T$. 
        \begin{theorem}\label{thm:hessian-ours}
        For any $\theta\in\RR^d$ and $J$ defined by (\ref{J_pi}), it holds that
        \begin{align}
        \nabla^2 J(\theta) = \EE\sum_{k=0}^{\infty} \gamma^k r_k \left[\nabla^2X(k;\tau)+\nabla X(k;\tau)\nabla X(k;\tau)^\tr \right].
        \end{align}
        \begin{proof}
            See \cref{sec:rl-derivatives}.
        \end{proof}
        \end{theorem}
        This formulation differs from the conventional expression derived from Maniyar et al. \cite{pmlr-v238-maniyar24a}:
        \begin{align}\label{old-hessian-esti}
            &\nabla^2 J(\theta)\notag\\
            &=\EE\sum_{k=0}^{\infty}
            \gamma^k r_k \left[\nabla^2 X(k;\tau)+
            \nabla X(k;\tau)\nabla X(T-1;\tau)^\tr\right].
        \end{align}
        The old Hessian matrix estimator derived from Expression (\ref{old-hessian-esti}) contains $\nabla X(T-1;\tau)$, which involves the summation of $T$ terms and is a variable with an upper bound of $\compU{T}$. In contrast, our novel Hessian matrix expression eliminates $\nabla X(T-1;\tau)$. As a result, the derived Hessian estimator's upper bound becomes independent of $T$, enabling application to MDPs of arbitrary length. This formulation offers broader applicability and a more straightforward analytical approach. 

        By leveraging \cref{thm:hessian-ours}, we establish upper bounds on the spectral norm of the Hessian matrix (representing the gradient Lipschitz constant). Furthermore, this formulation can be extended to derive the Lipschitz constant for the Hessian matrix of the original function, as we will formally demonstrate in \cref{lemma-Jall}. We next leverage \cref{thm:hessian-ours} to design estimators for the gradient and Hessian matrix.
        
        In practice, two modifications are required. First, the maximum length $H$ of sampled trajectories must be constrained to prevent infinite loops. Additionally, since computing the exact expectation would require enumerating all possible trajectories $\tau \sim p(\tau \mid \theta)$, we cannot compute the full gradient directly. Therefore, We resort to an empirical estimate of the gradient and Hessian matrix by sampling a dataset $\sT$ consisting of $|\sT|$ truncated trajectories $\tau_i =\{s_{i,0},a_{i,0},\cdots,s_{i,H}\}(0\leq i\leq |\sT|-1)$
        
        Thus, we can obtain estimators for the gradient and the Hessian matrix. It is noteworthy that both gradient and Hessian matrix estimators are the truncated versions of unbiased gradient estimators, and they are unbiased estimators of the gradient of the truncated expected return:
        \begin{align}\label{J_trun_vtheta}
        J_H(\theta) := \mathbb{E}_{\tau\sim \Pr(\cdot|\theta)}\left[\sum_{h=0}^{H-1} \gamma^h r_h\right]  
        \end{align}.

        \begin{corollary}
        Assume the dataset $\sT$ consist of $|\sT|$ trajectories $\tau\sim \Pr(\cdot|\theta)$, let
        \begin{align}
        g(\theta\mid\sT) = \frac{1}{|\sT|} \sum_{i=0}^{|\sT|-1} \sum_{k=0}^{H-1} \gamma^k r_{i,k} \nabla X(\tau_i,k)    \label{esti-g},
        \end{align} 
        \begin{align}
        &H(\theta|\sT)\notag\\
        &= \frac{1}{|\sT|} \sum_{\tau\in\sT} \sum_{k=0}^{H-1} \gamma^k r_{i,k} \left[\nabla^2 X(\tau_i,k)+\nabla X(\tau_i,k)\nabla X(\tau_i,k)^\tr\right].    \label{esti-H}             
        \end{align}    
        Then the gradient and the Hessian matrix of the truncated objective function (\ref{J_trun_vtheta}) are give by
        \begin{align}
            \EE g(\theta|\sT)&=\nabla J_H(\theta),\\
            \EE H(\theta|\sT)&=\nabla^2 J_H(\theta).
        \end{align}
        \end{corollary}
        
        Note that due to the truncation caused by the maximum horizon $H$, both of these estimators are biased. We will analyze their bias and the requirements for the value of $H$ in the subsequent sections.      
        \subsection{Cubic-Regularized Newton Method}

Second-order optimization methods leverage curvature information to escape saddle points and improve convergence over first-order methods. A classical approach is Newton’s method, which is updated as follows when solving problem (\ref{J_vtheta}):
\[
\theta_{t+1} = \theta_t - H(\theta_t)^{\dagger} g(\theta_t),
\]
where $g(\theta_t)$ is a gradient estimator and $H(\theta_t)$ approximates the Hessian. However, in reinforcement learning, $H(\theta_t)$ is often indefinite or ill-conditioned, invalidating the usual convergence guarantees.
To address this, the \emph{cubic-regularized Newton method}\cite{nesterov2006cubic} introduces a third-order regularization term to ensure global convergence. At each step, we solve:
\begin{align}\label{align:cubic-subproblem}
   \vh_t \in \argmin_{\vh \in \RR^d} \left\{ g(\theta_t)^\tr \vh + \frac{1}{2} \vh^\tr H(\theta_t) \vh + \frac{M}{6} \|\vh\|^3 \right\}, 
\end{align}
where $\vh=\theta_{t+1}-\theta_{t}$ is the parameter update vector, $M > 0$ is the cubic regularization coefficient. 
Compared with Newton's method, this algorithm can achieve global convergence (even for non-convex problems) under certain parameter settings, and the number of iterations required to reach $\epsilon$-FOSP is $\compU{\epsilon^{-1.5}}$, Faster than most first-order stochastic algorithms.
The reason is that under standard smoothness assumptions, when $M \geq L_3$, the cubic model upper bounds the true objective:
\begin{equation}\label{eq:upper-bound}
J(\theta_t + \vh) \leq J(\theta_t) + \nabla J(\theta_t)^\tr \vh + \frac{1}{2} \vh^\tr \nabla^2 J(\theta_t) \vh + \frac{M}{6} \|\vh\|^3.
\end{equation}
This ensures that as long as the gradient and Hessian estimators are accurate enough, minimizing the cubic model will lead to a decrease in the true objective, thus guaranteeing the global convergence of the algorithm.

It is worth mentioning that question (\ref{align:cubic-subproblem}) is a well-defined problem. Methods for solving it include matrix decomposition\cite{nesterov2006cubic, doikov2023second}, gradient descent\cite{carmon2019gradient}, subspace decomposition\cite{cartis2011adaptive}, etc. 
\subsection{Variance Reduction with Hessian Correction}

Variance reduction techniques aim to stabilize stochastic estimates without requiring additional samples. Methods like SVRG\cite{johnson2013accelerating} and STORM\cite{cutkosky2019momentum} achieve this by correcting gradient drift using past iterates.
In reinforcement learning, however, a key challenge arises: the data distribution evolves with the policy $\pi_\theta$, introducing distribution shift. This makes variance reduction more delicate, especially for policy gradient methods. Prior works (e.g., \cite{papini2018stochastic}) address this with importance sampling, but norm of the resulting weights can grow exponentially with $\|\theta_t - \theta_{t-1}\|$, leading to unbounded variance.

To avoid this, we adopt a second-order correction scheme based on Hessian-vector products. By Taylor expansion,
\[
\nabla J(\theta_t) - \nabla J(\theta_{t-1}) = \left[ \int_0^1 \nabla^2 J(\theta^{(\alpha)}) \, d\alpha \right] (\theta_t - \theta_{t-1}),
\]
where $\theta^{(\alpha)} := \alpha \theta_t + (1-\alpha) \theta_{t-1}$. We estimate this integral using stochastic Hessian estimators:
\[
\vg'_{t,b} := H(\theta_{t,b}'; \sT_{t,b}') \cdot (\theta_t - \theta_{t-1}), 
\]
where $\theta_{t,b}'=\theta_{t} - \alpha_b(\theta_t - \theta_{t-1}), \alpha_b \sim U(0,1)$.
This avoids importance sampling entirely, retains unbiasedness, and allows us to bound the estimator norm under standard smoothness assumptions.

Moreover, since $\theta_t - \theta_{t-1}$ is fixed, we can compute the Hessian-vector product directly, reducing the space complexity from $\mathcal{O}(d^2)$ to $\mathcal{O}(d)$.

        \subsection{Algorithm}
    By integrating the above techniques, we arrive at our proposed algorithm: the Variance-Reduced Cubic-Regularized Policy Newton method (VR-CR-PN). 

The proposed algorithm follows a single-layer iterative structure. In each iteration, it first estimates the policy gradient, then estimates the Hessian, and finally constructs and solves a cubic-regularized subproblem for policy update.

The gradient estimation alternates periodically between two strategies: at the beginning of each cycle, a large-batch estimator is used to obtain a low-variance baseline; in the remaining iterations, a Hessian-aided gradient estimator is employed to reduce variance by incorporating curvature information.

Following the gradient estimation, a small-batch stochastic Hessian estimator is constructed. After obtaining both the gradient and Hessian information, we formulate a cubic-regularized third-order subproblem, whose solution determines the update direction and step size.

When the step size $\vh_{t}$  produced by the algorithm at a certain iteration is sufficiently small, it indicates that $\theta_{t+1}$ is a second-order stationary point of the objective function satisfying the desired criteria. The algorithm then returns $\theta_{t+1}$ and terminates. Moreover, the properties of the cubic-regularized Newton method guarantee that within a finite number of iterations 
$T$, the algorithm will output a $\theta$ meeting these conditions.

The complete procedure is described in Algorithm~\ref{alg:VR-CR-PN}, where the $g(\theta \mid \sT)$ and $H(\theta \mid \sT)$ are defined in Equations~\eqref{esti-g} and~\eqref{esti-H}.
        \begin{algorithm}[tb]
		\caption{VR-CR-PN}
		\label{alg:VR-CR-PN}
		\begin{algorithmic}[1]
			\State {\bfseries Input:} Total iterations $T$, inner gradient
			length $S$, batch size $b_g$, $b_H$, batch size constant $B_g$, cubic regularization coefficient $M$, accuracy $\epsilon$, initial parameter $\theta_0 \in \RR^d$.
			\For{$t = 0, 1, \dots, T-1$}
			\If{$t \mod S = 0$}
                \State Sample a dataset $\sT_g$ consisting of $b_g$ trajectories following policy $\vpi_t$
			\State Compute $\vg_t = g(\theta_t | \sT_g)$
			\Else
                \State Set $b_{g,t}'=\lceil B_g\norm{\vh_{t-1}}^2\rceil$
                \For{$b = 0, 1, \dots, b_{g,t}'-1$}
                \State Randomly take $\alpha_b \in [0,1]$, $\theta_{t,b}'=\theta_{t} - \alpha_b\vh_{t-1}$. 
                \State Sample a dataset $\sT_{g,b}'$ consisting of one trajectories following policy $\vpi_{\theta_{t,b}'}$   
			\State Compute $\vg_{t,b}' = H(\theta_{t,b}' | \sT_{g,b}') \vh_{t-1}$
                \EndFor\\
			 Compute $\vg_{t} = \vg_{t-1} + \frac{1}{b_{g,t}'}\sum_{b=0}^{b_{g,t}'-1} \vg_{t,b}'$
			\EndIf
                \State Sample a dataset $\sT_H$ consisting of $b_H$ trajectories following policy $\vpi_t$
			\State Compute $\mH_t = H(\theta_t | \sT_H)$
			\State Solve $\vh_t \in \argmin_{\vh\in \RR^d}\vg_t^\tr \vh + \frac{1}{2} \vh^\tr \mH_t \vh + \frac{M}{6} \norm{\vh}^3$
			\State Set $\theta_{t+1} = \theta_t + \vh_t$
			\If{$\norm{\vh_t}\leq \sqrt{\epsilon/4L_3}$}
                \State {\bfseries return $\theta_{t+1}$}
			\EndIf        
			\EndFor
                \State {\bfseries return $\theta_T$}
		\end{algorithmic}
	\end{algorithm}
\section{Convergence Results}

    We now present theoretical convergence guarantees for the proposed VR-CR-PN algorithm. Our main result establishes that VR-CR-PN converges to an $\epsilon$-second-order stationary point with a sample complexity of $\widetilde{O}(\epsilon^{-3})$, under general smoothness assumptions. This improves upon the prior best result $\widetilde{O}(\epsilon^{-3.5})$ obtained by Maniyar et al.~\cite{pmlr-v238-maniyar24a} under similar conditions.
    To formalize our convergence results, we begin by recalling the standard notion of approximate stationarity. 

\begin{definition}[$\epsilon$-First-Order Stationary Point]
\label{def:FOSP}
Given $\epsilon > 0$, $\theta_R$ is said to be an $\epsilon$-first-order stationary point if
\[
\|\nabla J(\theta_R)\| \leq \epsilon.
\]
\end{definition}

To distinguish local minima from saddle points, we additionally consider second-order information. A point is a second-order stationary point if the gradient vanishes and the Hessian is positive semi-definite. In finite-time stochastic analysis, we relax this to:

\begin{definition}[$\epsilon$-Second-Order Stationary Point]
\label{def:SOSP}
Given $\epsilon > 0$, $\theta_R$ is said to be an $\epsilon$-second-order stationary point if there exists a constant $c > 0$ such that
\begin{align*}
\|\nabla J(\theta_R)\| \leq \epsilon, \quad
\lambda_{\min}(\nabla^2 J(\theta_R)) \geq -c \sqrt{\epsilon}.  
\end{align*}
\end{definition}
    In the stochastic optimization setting, we adopt high-probability analogues of these definitions, following standard practice in the literature~\cite{ghadimi2013stochastic}. 
        To derive explicit bounds for the policy gradient and the Hessian matrix, certain regularity assumptions are required on the Markov Decision Process (MDP) and the smoothness of the parameterized policy $\pi(\theta)$. Throughout, we use the $l_2$ norm for vectors and the spectral norm for matrices.

        \begin{assumption}(Bounded rewards)\label{a-0}
        \begin{align}
            \forall (s, a) \in \sS \times \sA, |r(s,a)|\leq R_{\max}.  
        \end{align}
        \end{assumption}
        
        \begin{assumption}(Bounded log-policy gradient)\label{a-1}
        \begin{align}
                \forall (s, a) \in \sS \times \sA, \norm{\nabla_\theta\log\pi(a|s)}\leq G_1.
        \end{align}
        \end{assumption}
        
        \begin{assumption}(Continuous log-policy gradient)\label{a-2}
        \begin{align}
            \forall (s, a) \in \sS \times \sA, \norm{\nabla_\theta^2\log\pi(a|s)}\leq G_2.
        \end{align}
        \end{assumption}
        
        \begin{assumption}(Smooth log-policy gradient)\label{a-3}
        \begin{align}
            \forall (s, a) \in \sS \times \sA, \forall \vy\in\RR^d, \norm{\nabla_\theta^3\log\pi(a|s)\vy}\leq G_3\norm{\vy}.
        \end{align}
        \end{assumption}
        \cref{a-0,a-1,a-2,a-3} are widely adopted in the study of policy gradient and actor-critic methods; see, e.g., \cite{shen2019hessian}. When using a Boltzmann policy parameterized linearly in features, these conditions are typically straightforward to verify (see \cref{app:loglinear-regularity}). Moreover, if the policy parameters are constrained within a compact domain, then by the continuity of the gradient and Hessian of the log-policy, Assumptions \ref{a-0} and \ref{a-1} follow directly.
        
        We note that Assumption \ref{a-3} is nearly equivalent to (yet slightly stronger than) the conventional Lipschitz Hessian assumption, with the only additional requirement being the existence of the $\nabla^3 \log\pi(\theta)$. This purpose-built assumption serves to streamline the proof procedure. Common policy parameterization methods, such as softmax and log-linear transformations, all satisfy \cref{a-0,a-1,a-2,a-3} (see Appendix \ref{app:loglinear-regularity} for details). 
        It should be emphasized that the classical Lipschitz Hessian assumption remains sufficient to establish \cref{thm:main}. The proof methodology would closely follow that presented in \cref{sec:rl-derivatives}, with the only modification being the replacement of the conclusion in Lemma \ref{lemma-J3} with the Lipschitz continuity of $\nabla^2 J(\theta)$.
        \begin{lemma}\label{lemma-Jall}
            Under Assumptions \ref{a-0}, \ref{a-1}, \ref{a-2} and \ref{a-3}, for any $\theta, \vy\in \RR^d$ and trajectories set $\sT$, we have
        \begin{align}
            & |J(\theta)|\leq L_0,\label{lemma-J0}\\
            & \norm{\nabla J(\theta)} \leq L_1,\quad \norm{\vg(\sT|\theta)} \leq L_1,\label{lemma-J1}\\
            & \norm{\nabla^2 J(\theta)} \leq L_2,\quad \norm{\mH(\sT|\theta)} \leq L_2,\label{lemma-J2}\\
            & \norm{\nabla^3 J(\theta)\vy}\leq L_3 \norm{\vy}.\label{lemma-J3}
        \end{align}
        where 
        \begin{align*}
            L_0&:=\frac{R_{\max}}{1-\gamma},\quad
            L_1:=\frac{G_1 R_{\max}}{(1-\gamma)^{2}},\\
            L_2&:=\frac{G_2R_{\max}}{(1-\gamma)^2}
            +\frac{2G_1^2R_{\max}}{(1-\gamma)^3},\\
            L_3&:=\frac{G_3R_{\max}}{(1-\gamma)^2}
            +\frac{(6G_2G_1+2G_1^3)R_{\max}}{(1-\gamma)^3}.
        \end{align*}
        \end{lemma}
        \begin{proof}
            See Appendix \ref{app:grad-hess-bounds}, \ref{app:third-derivative}.
        \end{proof} 
\begin{theorem}
\label{thm:main}
Let $\theta_{\text{out}}$ be the output of the VR-CR-PN algorithm. Suppose \cref{a-0,a-1,a-2,a-3} hold. Set:
\begin{itemize}
    \item Cubic regularization coefficient: $M = 30L_3$,
    \item Total number of iterations: $T \geq 7L_0 L_3^{1/2} \epsilon^{-3/2}$,
    \item Inner loop length: 
    $S = L_1L_2^{-1}L_3^{1/2}\epsilon^{-1/2}$,
    \item Batch sizes:
    \begin{align*}
        b_g &= 2592 L_1^2 \left[\log\left(\tfrac{3T}{P}\right)\right]^2 \epsilon^{-2}, \\
        b_H &= 144 L_2^2 L_3^{-1} \log\left(\tfrac{3Td}{P}\right) \epsilon^{-1}, \\
        B_g &= 2592 L_2^2 \left[\log\left(\tfrac{3T}{P}\right)\right]^2 S \epsilon^{-2},
    \end{align*}
    \item Truncation horizon $H$ chosen according to \cref{lem:truncation-bias}.
\end{itemize}

Then, with probability at least $1 - P$, the output $\theta_{\text{out}}$ is an $6\epsilon$-second-order stationary point, i.e.,
\[
\| \nabla J(\theta_{\text{out}}) \| \leq 6\epsilon, \quad
\lambda_{\min} \left( \nabla^2 J(\theta_{\text{out}}) \right) \geq -9\sqrt{L_3 \epsilon}.
\]
\end{theorem}

\cref{thm:main} establishes that VR-CR-PN converges to an $\epsilon$-second-order stationary point with high probability under mild smoothness assumptions. 

        \begin{corollary}[Total Sample Complexity]
        \label{cor:total-sample}
        Under the parameter setting specified in Theorem~\ref{thm:main}, the VR-CR-PN algorithm achieves $\epsilon$-second-order stationarity with high probability using the following total number of samples:
        \[
        \text{Gradient: } \quad \tilde{\mathcal{O}}(\epsilon^{-3}), \qquad
        \text{Hessian: } \quad \tilde{\mathcal{O}}(\epsilon^{-3}).
        \]
        \end{corollary}

        Here, the notation $\tilde{\mathcal{O}}$ hides polynomial dependencies on $\log(1/\epsilon)$ and $\log d$. These bounds follow from the analysis in \cref{proof:cor-total-sample}.
        \textbf{\cref{cor:total-sample} shows that VR-CR-PN achieves the best known sample complexity among second-order methods in reinforcement learning, up to logarithmic factors.}
        
        This work establishes a theoretical breakthrough in reinforcement learning. Under general non-convexity assumptions, the proposed VR-CR-PN algorithm converges to an $\epsilon$-second-order stationary point using only $\tilde{\mathcal{O}}(\epsilon^{-3})$ samples. This represents a substantial improvement over prior state-of-the-art methods, such as CR-PN~\cite{pmlr-v238-maniyar24a}, which require $\tilde{\mathcal{O}}(\epsilon^{-3.5})$ samples. 
        
\section{Simulation Experiments}
\label{sec:experiments}
        We compare our method against the CR-PN algorithm, which applies cubic regularization in policy optimization without employing any variance reduction mechanisms.
        The experimental results show that the VR-CR-PN method significantly outperforms traditional policy gradient methods in terms of convergence speed and stability.

        The experimental environment employed in our study is CartPole-v1, which is a classic benchmark environment in reinforcement learning that models an inverted pendulum control problem. In this task, an agent must learn to balance a pole vertically atop a movable cart by applying discrete forces to the cart's base. The state space $\sS$ consists of four continuous variables: cart position $x$, cart velocity $\dot{x}$, pole angle $\omega$, and pole angular velocity $\dot{\omega}$. The agent receives a positive reward for each timestep the pole remains upright, with an episode terminating when the pole deviates beyond $\pm12^\circ$ from vertical or the cart moves outside the predefined bounds. 
        At each timestep prior to termination, the agent receives a unit reward $r = 1$, whereas upon failure it receives zero reward $r = 0$ and the episode terminates. This environment configuration clearly satisfies Assumption \ref{a-0}.
        
        The environment's simplicity and well-defined dynamics make it particularly suitable for evaluating fundamental aspects of policy optimization algorithms, especially in terms of sample efficiency and learning stability. 
        The action space for any non-terminal state is defined as 
        $\sA=\{0,1\}$, where action 0 corresponds to moving the CartPole system to the left and action 1 to the right.
        To model the policy distribution of the CartPole system, we employ a log-linear parameterization approach:
        \begin{align}
            \pi_{\theta}(a|s)=\frac{\exp(\vs^\top\theta_a)}{\sum_{a'\in\sA}\exp(\vs^\top\theta_{a'})
            },
        \end{align}
        where $\vs \in \RR^4$ denotes the vectorized representation of the state space $\sS$. The dimensionality of the policy parameter vector $\theta$ is determined by the product of the state space cardinality and action space cardinality, specifically $|\sS|\times|\sA|=8$.
        when the parameter vector $\theta$ is bounded, this particular policy parameterization can be rigorously shown to satisfy Assumptions \ref{a-1} through \ref{a-3}. For a rigorous derivation of these results, we refer the reader to Appendix \ref{app:loglinear-regularity}, where the complete mathematical proofs are presented in detail.

        To ensure the subproblems in our optimization procedure consistently yield values greater than the objective function, we carefully selected a sufficiently large cubic regularization coefficient $M$ based on the discount factor $\gamma$. Due to the inherent challenges in reinforcement learning, where precise computation of gradient norms and Hessian eigenvalues proves intractable, we follow the methodology established in prior work \cite{pmlr-v238-maniyar24a} by instead comparing estimated function values.
        For rigorous experimental comparison, both algorithms were configured with identical upper bounds on sample size (50,000 samples). 
	\begin{figure}[tp!]
            \centering
            \includegraphics[width=\linewidth]{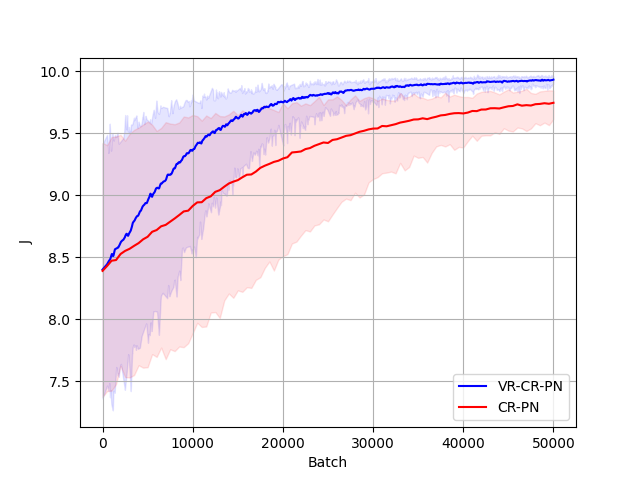}
            \caption{Experimental Results}
            \label{fig:screenshot001}
	\end{figure}  
        
        Given the non-negligible variance in function value sampling, we adopted two key measures to enhance estimation reliability: first, we increased the number of samples per episode through multiple experimental repetitions; second, after collecting 50,000 samples, we performed a final function value estimation using a significantly larger sample size to reduce variance. This two-phase approach balances computational efficiency with estimation accuracy.
      
        The comparative analysis in \cref{fig:screenshot001} reveals that the VR-CR-PN algorithm maintains a consistent performance advantage of 0.25-0.5 over CR-PN across almost all batch sizes (50,00-50,000), while exhibiting markedly reduced outcome variability as evidenced by the narrower semi-transparent bands.
        
        The sustained performance gap between variance-reduced and standard implementations suggests the proposed control mechanisms effectively mitigate stochastic fluctuations while accelerating convergence, consistent with the prior algorithmic performance analysis of VR-CR-PN. This observed behavior aligns with theoretical expectations for variance-reduced policy optimization. The results demonstrate VR-CR-PN's dual advantages in both central tendency and stability - particularly valuable for applications requiring predictable convergence behavior.
        
\section{Conclusions}
        In this work, we present a variance-reduced cubic-regularized policy Newton method (VR-CR-PN) accompanied by a novel Hessian estimator. 
        Our algorithm employs a specialized variance reduction technique to achieve the best-known sample complexity for convergence to second-order stationary points under standard assumptions. Furthermore, the proposed Hessian estimator attains a spectral norm bound that is independent of the truncation horizon $H$.
        Rigorous theoretical analysis establishes the algorithm's convergence to second-order stationary points of the value function - a property that guarantees escape from saddle points. Notably, the proposed method achieves an improved sample complexity of $\widetilde{O}(\epsilon^{-3})$, representing an $\widetilde{O}(\epsilon^{-0.5})$ reduction compared to the current state-of-the-art $\widetilde{O}(\epsilon^{-3.5})$. To our knowledge, this constitutes the best sample complexity under equivalent assumptions in the literature.

\bibliographystyle{IEEEtran}
\bibliography{references}

\section{Useful Lemmas}
\label{app:lemmas}

\subsection{Differentiation Under Expectation}
\label{sec:diff-expectation}
We establish a general result for differentiating expectations under parameterized distributions, which is useful for computing policy gradients in reinforcement learning.

\begin{lemma}[Differentiation of Expectations]
\label{lem:diff-expectation}
Let $Y$ be a random variable drawn from a distribution $p(y; \theta)$ that is differentiable with respect to $\theta$, and let $f(Y)$ be a differentiable function. Then the gradient of the expectation of $f(Y)$ with respect to $\theta$ is
\begin{align*}
&\nabla \EE_{Y \sim p(\cdot; \theta)} f(Y) \\
&= \EE_{Y \sim p(\cdot; \theta)} \left[ \nabla f(Y) + f(Y) \otimes \nabla \log p(Y; \theta) \right],    
\end{align*}

where $\nabla$ denotes the gradient with respect to $\theta$, and $\otimes$ denotes the tensor product between $f(Y)$ and $\nabla \log p(Y; \theta)$.
\end{lemma}

\begin{proof}
Using the product rule and the chain rule:
\begin{align*}
&\nabla \EE_{Y \sim p(\cdot; \theta)} f(Y)
= \nabla \sum_{Y} p(Y; \theta) f(Y) \\
&= \sum_{Y} \left[ \nabla p(Y; \theta) \cdot f(Y) + p(Y; \theta) \cdot \nabla f(Y) \right] \\
&= \sum_{Y} p(Y; \theta) \left[ \nabla f(Y) + f(Y) \otimes \nabla \log p(Y; \theta) \right] \\
&= \EE_{Y \sim p(\cdot; \theta)} \left[ \nabla f(Y) + f(Y) \otimes \nabla \log p(Y; \theta) \right].
\end{align*}
\end{proof}

In reinforcement learning, applying Lemma~\ref{lem:diff-expectation} to trajectories $\tau \sim \Pr(\cdot\mid\theta)$ gives:
\[
\nabla \EE_{\tau \sim \pi_\theta} f(\tau) = \EE_{\tau \sim \pi_\theta} \left[ \nabla f(\tau) + f(\tau) \otimes \nabla X(T; \tau) \right].
\]

When $f(\tau)$ depends only on the prefix of $\tau$ up to state $s_t$, then we can just add up the log-gradients up to step $t-1$:
\[
\nabla \EE_{\tau \sim \pi_\theta} f(\tau^{[t]}) = \EE_{\tau \sim \pi_\theta} \left[ \nabla f(\tau^{[t]}) + f(\tau^{[t]}) \otimes \nabla X(t; \tau) \right].
\]
\subsection{Azuma–Hoeffding Inequality}
We invoke two concentration inequalities that generalize the classical Azuma--Hoeffding inequality to vector-valued and matrix-valued martingale difference sequences.

\begin{lemma}[Vector Azuma Inequality {\cite{pinelis1994optimum}}]
\label{lem:vector-azuma}
Let $\{\vv_k\}$ be a vector-valued martingale difference sequence adapted to a filtration, such that $\EE[\vv_k \mid \sigma(\vv_1, \dots, \vv_{k-1})] = \bs{0}$ and $\|\vv_k\|_2 \leq A_k$ almost surely. Then with probability at least $1 - \delta$,
\[
\left\| \sum_k \vv_k \right\|_2 \leq 3 \sqrt{ \log(1/\delta) \sum_k A_k^2 }.
\]
\end{lemma}
\begin{lemma}[Matrix Azuma Inequality {\cite{tropp2011freedman}}]
\label{lem:matrix-azuma}
Let $\{\mX_k\}$ be a finite adapted sequence of self-adjoint matrices in dimension $d$, and let $\{\mA_k\}$ be a fixed sequence of self-adjoint matrices satisfying
\[
\EE[\mX_k \mid \sigma(\mX_1, \dots, \mX_{k-1})] = \bs{0}, \quad \text{and} \quad \mX_k^2 \preceq \mA_k^2 \quad \text{almost surely}.
\]
Then with probability at least $1 - \delta$,
\[
\left\| \sum_k \mX_k \right\|_2 \leq 3 \sqrt{ \log(d/\delta) \sum_k \| \mA_k \|_2^2 }.
\]
\end{lemma}

\subsection{Solutions to Logarithmic Inequalities}

\begin{lemma}
\label{lem:log-ineq}
Let $a, b > 0$, and $K \geq \max\{ a,\, 2a \log(2ab) \}$. Then
\[
a \log(bK) \leq K.
\]
\end{lemma}

\begin{proof}
Define $K' := 2a \log(2ab)$. Then
\[
a \log(bK') = a \log(2ab \log(2ab)) < 2a \log(2ab) = K'.
\]
For all $K > K'$, we have $(a \log(bK) - K)' = a/K - 1 < 0$, so the inequality holds.
\end{proof}

\subsection{Optimality Conditions of the Cubic Subproblem}
\label{sec:o-cnm}
\begin{lemma}\cite{carmon2019gradient}
For any vector $\vg$ and symmetric matrix $\mH$, let $\vh_t$ be a global minimizer of the cubic model
\begin{align}\label{align:cubic-subproblem-2}
m_t(\vh) := \NP{\vg_t}{\vh} + \frac{1}{2} \vh^\tr \mH_t \vh + \frac{M}{6} \norm{\vh}^3.    
\end{align}

Then the following conditions hold:
\begin{align}
\label{eq:cubic-optimality-gradient}
\vg_t + \mH_t \vh_t + \frac{M}{2} \norm{\vh_t} \vh_t &= \bs{0}, \\
\label{eq:cubic-optimality-hessian}
\mH_t + \frac{M}{2} \norm{\vh_t} \mI &\succeq \bs{0}.
\end{align}
\end{lemma}
It is immediate to verify that $m_t(\vh_t) \leq m_t(\bs{0}) = 0$. Moreover, a sharper upper bound on $m_t(\vh_t)$ can be established:
\begin{lemma}\cite{carmon2019gradient}
For any vector $\vg$ and symmetric matrix $\mH$, let $\vh_t$ be a global minimizer of the cubic model (\ref{align:cubic-subproblem-2}), then the following inequality holds:
\begin{equation}
\label{eq:model-descent}
m_t(\vh_t) \leq - \frac{M}{12} \norm{\vh_t}^3.
\end{equation}    
\end{lemma}

\section{Proof of \cref{thm:hessian-ours}}
\label{sec:rl-derivatives}

In this section, we derive the first- and second-order derivatives of the expected return $J(\theta)$ in reinforcement learning. These results support the theoretical analysis of our algorithm. The function $J(\theta)$, the log-likelihood trace $X(k; \tau)$, and the constants $R_{\max}, G_1, G_2, G_3$ have been defined in the main text.
\label{sec:grad-hess-rl}

Recall that the objective function is
\[
J(\theta) = \EE_{\tau \sim \pi_\theta} \left[ \sum_{k=0}^\infty \gamma^k r_k \right].
\]
Applying Lemma~\ref{lem:diff-expectation}, the gradient becomes:
\begin{align*}
\nabla J(\theta)
&= \nabla \EE \sum_{k=0}^\infty \gamma^k r_k
= \sum_{k=0}^\infty \nabla \EE \left[ \gamma^k r_k \right] \\
&= \sum_{k=0}^\infty \EE \left[ \nabla \left( \gamma^k r_k \right) + \gamma^k r_k \cdot \nabla X(k; \tau) \right] \\
&= \sum_{k=0}^\infty \EE \left[ \gamma^k r_k \cdot \nabla X(k; \tau) \right],
\end{align*}
where the last step uses the fact that $r_k$ is not a function of $\theta$ under the standard stochastic MDP model.

Therefore, the policy gradient is:
\[
\nabla J(\theta) = \EE_{\tau \sim \pi_\theta} \left[ \sum_{k=0}^\infty \gamma^k r_k \cdot \nabla X(k; \tau) \right].
\]

For the Hessian, we compute:
\begin{align*}
&\nabla^2 J(\theta)\\
&= \nabla \EE \sum_{k=0}^\infty \gamma^k r_k \cdot \nabla X(k; \tau)
= \sum_{k=0}^\infty \nabla \EE \left[ \gamma^k r_k \cdot \nabla X(k; \tau) \right] \\
&= \sum_{k=0}^\infty \EE \left[
\gamma^k r_k \cdot \left( \nabla^2 X(k; \tau) + \nabla X(k; \tau) \cdot \nabla X(k; \tau)^\top \right)
\right],
\end{align*}
which yields the second-order formula:
\begin{align*}
&\nabla^2 J(\theta) \\
&= \EE \left[ \sum_{k=0}^\infty \gamma^k r_k \cdot \left( \nabla^2 X(k; \tau) + \nabla X(k; \tau) \nabla X(k; \tau)^\top \right) \right].    
\end{align*}

These formulas provide the basis for computing or estimating the gradient and Hessian of the objective using trajectories generated from the policy $\pi_\theta$.

\section{Proof of \cref{lemma-Jall}}
\subsection{Boundedness of Gradient and Hessian Estimators}
\label{app:grad-hess-bounds}
We now establish uniform norm bounds for the sample-based gradient and Hessian estimators $\vg(\sT \mid \theta)$ and $\mH(\sT \mid \theta)$, as introduced in Section~2.1. These estimators are computed using rollouts truncated at $H$ steps and are averaged over a trajectory set $\sT$. The bounds rely on the regularity assumptions stated in Assumption~2.1 and imply boundedness of the exact derivatives $\nabla J(\theta)$ and $\nabla^2 J(\theta)$ by virtue of estimator unbiasedness.

Recall that for any trajectory $\tau = (s_0, a_0, s_1, a_1, \dots)$, the log-policy accumulation and its derivatives are denoted as:
\begin{align*}
&X(t; \tau) := \sum_{k=0}^{t} \log \pi(a_k \mid s_k), \\
&\nabla X(t; \tau) = \sum_{k=0}^{t} \nabla \log \pi(a_k \mid s_k), \\
&\nabla^2 X(t; \tau) = \sum_{k=0}^{t} \nabla^2 \log \pi(a_k \mid s_k).    
\end{align*}

As defined in the main text, the estimators $\vg(\sT \mid \theta)$ and $\mH(\sT \mid \theta)$ evaluated within $H$ truncated steps over a set $\sT$ of trajectories are given by:
\begin{align*}
&\vg(\sT \mid \theta) := \frac{1}{|\sT|} \sum_{\tau \in \sT} \sum_{k=0}^{H-1} \gamma^k r_k \nabla X(k; \tau), \\
&\mH(\sT \mid \theta) \\
&:= \frac{1}{|\sT|} \sum_{\tau \in \sT} \sum_{k=0}^{H-1} \gamma^k r_k \left( \nabla^2 X(k; \tau) + \nabla X(k; \tau) \nabla X(k; \tau)^\top \right). 
\end{align*}

\paragraph{Gradient Estimator.}  
Using the assumptions $\| \nabla \log \pi(a \mid s) \| \leq G_1$ and $|r_k| \leq R_{\max}$, we obtain:
\begin{align*}
\| \vg(\sT \mid \theta) \|
&\leq \sum_{k=0}^{H-1} \gamma^k |r_k| \| \nabla X(k; \tau) \|\\
&\leq \sum_{k=0}^{H-1} \gamma^k R_{\max} (k+1) G_1 \\
&= G_1 R_{\max} \sum_{k=0}^{H-1} \gamma^k (k+1)\\
&\leq G_1 R_{\max} \sum_{k=0}^{\infty} \gamma^k (k+1)
= \frac{G_1 R_{\max}}{(1 - \gamma)^2}.
\end{align*}

\paragraph{Hessian Estimator.}  
By applying the triangle inequality and the bounds $\| \nabla^2 \log \pi(a \mid s) \| \leq G_2$, we get:
\begin{align*}
\| \mH(\sT \mid \theta) \|
&\leq \sum_{k=0}^{H-1} \gamma^k |r_k| \left( \| \nabla^2 X(k; \tau) \| + \| \nabla X(k; \tau) \|^2 \right) \\
&\leq \sum_{k=0}^{H-1} \gamma^k R_{\max} \left( G_2 (k+1) + G_1^2 (k+1)^2 \right) \\
&= G_2 R_{\max} \sum_{k=0}^{H-1} \gamma^k (k+1) + G_1^2 R_{\max} \sum_{k=0}^{H-1} \gamma^k (k+1)^2 \\
&\leq \frac{G_2 R_{\max}}{(1 - \gamma)^2} + \frac{2 G_1^2 R_{\max}}{(1 - \gamma)^3}.
\end{align*}

\paragraph{Implication for True Gradient and Hessian.}  
Since the estimators $\vg(\sT \mid \theta)$ and $\mH(\sT \mid \theta)$ are unbiased, i.e., 
\[
\mathbb{E}[\vg(\sT \mid \theta)] = \nabla J(\theta), \quad
\mathbb{E}[\mH(\sT \mid \theta)] = \nabla^2 J(\theta),
\]
we conclude via Jensen's inequality:
\begin{align*}
\| \nabla J(\theta) \| &\leq \mathbb{E}[\| \vg(\sT \mid \theta) \|] \leq \frac{G_1 R_{\max}}{(1 - \gamma)^2}, \\
\| \nabla^2 J(\theta) \| &\leq \mathbb{E}[\| \mH(\sT \mid \theta) \|] \leq \frac{G_2 R_{\max}}{(1 - \gamma)^2} + \frac{2 G_1^2 R_{\max}}{(1 - \gamma)^3}.
\end{align*}

\subsection{Third Derivative of $J(\theta)$ and Its Norm Bound}
\label{app:third-derivative}

In this section, we derive the expression for the third-order directional derivative of the expected return objective $J(\theta)$, and establish a norm upper bound for $\nabla^3 J(\theta)[\vy]$ for any vector $\vy$ with $\|\vy\| = 1$.

\paragraph{Expression.} From the result in \cref{thm:hessian-ours}, we recall:
\[
\nabla^2 J(\theta)
= \EE\sum_{k=0}^\infty \gamma^k r_k \left( \nabla^2 X(k; \tau) + \nabla X(k; \tau) \nabla X(k; \tau)^\tr \right).
\]
Taking a directional derivative in direction $\vy$, we get:
\begin{align*}
\nabla^3 J(\theta)[\vy]
&= \nabla \left[ \nabla^2 J(\theta) \cdot \vy \right] \\
&= \EE\sum_{k=0}^\infty \gamma^k r_k \Big[
\nabla^3 X(k; \tau)[\vy]\\
&\quad+ (\vy^\tr \nabla X(k; \tau)) \nabla^2 X(k; \tau) \\
&\quad + \nabla X(k; \tau) (\nabla^2 X(k; \tau)[\vy])^\tr \\
&\quad + \nabla^2 X(k; \tau)[\vy] \nabla X(k; \tau)^\tr \\
&\quad + (\vy^\tr \nabla X(k; \tau)) \nabla X(k; \tau) \nabla X(k; \tau)^\tr
\Big].
\end{align*}
\paragraph{Simplification via Expectation.}
To bound the final cubic term 
\[
(\vy^\tr \nabla X(k; \tau)) \nabla X(k; \tau) \nabla X(k; \tau)^\tr,
\]
we observe that although its norm can grow as $\mathcal{O}((k+1)^3)$ in the worst case, its expectation can be upper bounded by $\mathcal{O}((k+1)^2)$ using the independence structure of the trajectory. Specifically,
we use the fact that the cross terms vanish in expectation. For any $i < j$, we have:
\[
\EE \left[ \nabla \log \pi(a_i \mid s_i)^\tr \nabla \log \pi(a_j \mid s_j) \right] = 0,
\]
because $\EE_{a_j \sim \pi(\cdot \mid s_j)} \nabla \log \pi(a_j \mid s_j) = 0$. Therefore:
\begin{align*}
&\EE \left[ \nabla X(k; \tau) \nabla X(k; \tau)^\tr \right]\\
&= \sum_{i=0}^k \EE \left[ \nabla \log \pi(a_i \mid s_i) \nabla \log \pi(a_i \mid s_i)^\tr \right] 
\preceq (k+1) G_1^2 \cdot \mI \quad \\
&\Rightarrow \quad \left\| \EE \left[ \nabla X(k; \tau) \nabla X(k; \tau)^\tr \right] \right\| \leq (k+1) G_1^2.
\end{align*}
Combining these observations, we obtain:
\begin{align*}
&\left\| \EE \left[ (\vy^\tr \nabla X(k; \tau)) \nabla X(k; \tau) \nabla X(k; \tau)^\tr \right] \right\|\\
&= \left\| \EE \left[ (\vy^\tr \nabla X(k; \tau)) \cdot \nabla X(k; \tau) \nabla X(k; \tau)^\tr \right] \right\| \\
&\leq \left\|\vy^\tr \nabla X(k; \tau)\right\|_{\max} \cdot \left\| \EE \left[ \nabla X(k; \tau) \nabla X(k; \tau)^\tr \right] \right\| \\
&\leq (k+1)^2 G_1^3.
\end{align*}
Thus, we replace the cubic bound $(k+1)^3 G_1^3$ with a tighter expectation-based bound of $(k+1)^2 G_1^3$.

\paragraph{Bounding Each Term.}
We now upper bound the norm of each term in $\nabla^3 J(\theta)[\vy]$ under the \cref{a-0,a-1,a-2,a-3}:
\begin{align*}
|r_k| \leq R_{\max}&, \quad
\|\nabla \log \pi(a \mid s)\| \leq G_1, \\
\|\nabla^2 \log \pi(a \mid s)\| \leq G_2&, \quad
\|\nabla^3 \log \pi(a \mid s)[\vy]\| \leq G_3.    
\end{align*}

Using the fact that $X(k; \tau)$ is the sum of $k+1$ terms of the form $\log \pi(a_i \mid s_i)$, we obtain:
\begin{align*}
&\|\nabla^3 J(\theta)[\vy]\|\\
&\leq \sum_{k=0}^\infty \gamma^k R_{\max} \cdot \Big[
(k+1) G_3
+ 3 (k+1)^2 G_1 G_2
+ (k+1)^2 G_1^3
\Big] \\
&= R_{\max} \left[
\frac{G_3}{(1 - \gamma)^2}
+ \frac{6 G_1 G_2 + 2 G_1^3}{(1 - \gamma)^3}
\right].
\end{align*}

\section{Proof of \cref{thm:main}}
\subsection{Cubic Newton Method Analysis}
\label{sec:convergence}
This subsection provides the technical lemmas and detailed convergence analysis.
The proofs rely on inequalities (\ref{eq:grad-error-bound})(\ref{eq:hess-error-bound}) and support \cref{thm:main} in the main text. The analysis builds upon the classical cubic regularization framework but extends it to the stochastic setting with variance-reduced estimators. 

Our goal is to show that the output iterate $\theta_{\text{out}}$ of VR-CR-PN satisfies the second-order stationarity conditions with high probability. The proof distinguishes two regimes based on the step size $\norm{\vh_t}$: either $\vh_t$ is large, leading to significant objective descent, or $\vh_t$ is small, indicating approximate stationarity.

Suppose the algorithm parameters—including $H$, $b_g$, $B_g$, $b_H$, $T$, $S$ and $M$—are selected as in \cref{thm:main}. Then, by combining the deviation bounds in \cref{sec:variance-reduction} (\cref{lem:combined-prob}) with the truncation bounds in \cref{sec:truncation} (\cref{lem:truncation-bias}), we conclude that with probability at least $1 - P$, the following hold for all $t$:
\begin{align}
\norm{ \vg_t - \nabla J(\theta_t) } &\leq \epsilon, \label{eq:grad-error-bound} \\
\norm{ \mH_t - \nabla^2 J(\theta_t) } &\leq \sqrt{L_3 \epsilon}. \label{eq:hess-error-bound}
\end{align}

These follow by triangle inequality:
\begin{align*}
\norm{ \vg_t - \nabla J(\theta_t) } &\leq \norm{ \vg_t - \nabla J_H(\theta_t) } + \norm{ \nabla J_H(\theta_t) - \nabla J(\theta_t) } \\
&\leq \frac{\epsilon}{2} + \frac{\epsilon}{2} = \epsilon. 
\end{align*}
\begin{align*}
&\norm{ \mH_t - \nabla^2 J(\theta_t) } \\
&\leq \norm{\mH_t - \nabla^2 J_H(\theta_t)} + \norm{\nabla^2 J_H(\theta_t) - \nabla^2 J(\theta_t)} \\
&\leq \frac{1}{2} \sqrt{L_3\epsilon} + \frac{1}{2} \sqrt{L_3\epsilon} = \sqrt{L_3\epsilon}.
\end{align*}

We first consider the case where the norm $\norm{\vh_t}$ is large. In this regime, we quantify the decrease in the true objective function $J$ based on cubic regularization theory and smoothness assumptions. We then use this quantitative bound to establish that such descent cannot persist indefinitely. This leads to the existence of an iterate with a small step norm, which will be analyzed next.

\subsubsection{Function Decrease in the Large-Step Regime}
In this subsection, we make use of the optimality conditions of the cubic subproblem, summarized in \cref{sec:o-cnm}.

First, we relate the model decrease to the actual objective value. Let $\theta_{t+1} = \theta_t + \vh_t$. Using the Taylor expansion upper bound for $J$:
\[
J(\theta_{t+1}) \leq J(\theta_t) + \NP{ \nabla J(\theta_t) }{ \vh_t } + \frac{1}{2} \vh_t^\tr \nabla^2 J(\theta_t) \vh_t + \frac{L_3}{6}\norm{\vh_t}^3.
\]
Define the deviation:
\[
\Delta_t := J(\theta_{t+1}) - \left( J(\theta_t) + m_t(\vh_t) \right).
\]
Using the estimator error bounds, we write:
\begin{align*}
\Delta_t 
&= \left( \nabla J(\theta_t) - \vg_t \right)^\tr \vh_t + \frac{1}{2} \vh_t^\tr \left( \nabla^2 J(\theta_t) - \mH_t \right) \vh_t \\
&+ \frac{L_3-M}{6}\norm{\vh_t}^3\\
&\leq \epsilon \norm{\vh_t} + \frac{\sqrt{L_3\epsilon}}{2} \norm{\vh_t}^2 + \frac{L_3-M}{6}\norm{\vh_t}^3.
\end{align*}

Combining with \eqref{eq:model-descent}, we obtain:
\begin{align*}\label{eq:actual-descent}
&J(\theta_t) - J(\theta_{t+1}) \\
&> \left(\frac{M}{4}-\frac{L_3}{6}\right) \norm{\vh_t}^3 - \epsilon \norm{\vh_t} - \frac{\sqrt{L_3\epsilon}}{2} \norm{\vh_t}^2 \\
&\geq\left(\frac{M}{4}-\frac{L_3}{6}-\frac{4L_3}{3}-\frac{2L_3}{3}\right) \norm{\vh_t}^3-\left(\frac{\epsilon^{1.5}}{3L_3^{0.5}}+\frac{\epsilon^{1.5}}{24L_3^{0.5}}\right)\\
&=\left(\frac{M}{4}-\frac{13L_3}{6}\right)\norm{\vh_t}^3-\frac{3\epsilon^{1.5}}{8L_3^{0.5}}.    
\end{align*}
If $\norm{\vh_t} > \frac{1}{2}\sqrt{\frac{\epsilon}{L_3}}$, then the final part is positive with $M = 30L_3$, and we obtain
\[
J(\theta_t) - J(\theta_{t+1}) > 
\frac{7}{24}L_3^{-0.5}\epsilon^{1.5},
\]
Since $J$ is bounded below by $J^*$ and has range at most $2L_0$, if we run the algorithm for
\[
T \geq 7L_0L_3^{0.5}\epsilon^{-1.5},
\]
then the descent bound cannot hold at all $T$ steps. Therefore, there must exist $t^* \in \{0, 1, \dots, T-1\}$ such that
\[
\norm{\vh_{t^*}} \leq \frac{1}{2}\sqrt{\frac{\epsilon}{L_3}}.
\]

We will show that the corresponding iterate $\theta_{t^*+1}$ satisfies the $\epsilon$-second-order stationarity condition in the next.

\subsubsection{Second-Order Stationarity at $\theta_{t+1}$ under Small Step}
We now show that when the step norm $\norm{\vh_t}$ is sufficiently small, the next iterate $\theta_{t+1} = \theta_t + \vh_t$ satisfies the $\epsilon$-second-order stationarity condition. The proof refers to the work of Carmon et al.\cite{carmon2019gradient}.
\begin{lemma}
Assuming inequalities (\ref{eq:grad-error-bound}) and (\ref{eq:hess-error-bound}) hold, $\vh_t$ satisfies the equation (\ref{align:cubic-subproblem}), and $M = 30L_3$, then $\theta_{t+1} = \theta_t + \vh_t$ is a $6\epsilon$-SOSP of $J$, satisfying:
\[
\| \nabla J(\theta_{t+1}) \| \leq 6\epsilon, \quad
\lambda_{\min} \left( \nabla^2 J(\theta_{t+1}) \right) \geq -9\sqrt{L_3 \epsilon}.
\]
\end{lemma}
\begin{proof}
Using the Taylor expansion of the gradient upper bound:
\[
\norm{ \nabla J(\theta_{t+1}) } \leq \norm{ \nabla J(\theta_t) + \nabla^2 J(\theta_t) \vh_t }+\frac{L_3}{2}\norm{\vh_t}^2.
\]
From the first-order optimality condition of the cubic subproblem (\ref{eq:cubic-optimality-gradient}), we recall:
\[
\vg_t + \mH_t \vh_t + \frac{M}{2} \norm{\vh_t} \vh_t = \bs{0},
\]
which implies
\begin{align*}
    &\nabla J(\theta_t)+\nabla^2 J(\theta_t)\vh_t \\
    &= - \frac{M}{2} \norm{\vh_t} \vh_t + (\nabla J(\theta_t) - \vg_t) + (\nabla^2 J(\theta_t)-\mH_t) \vh_t.
\end{align*}

Substituting:
\[
\norm{ \nabla J(\theta_{t+1}) } \leq \frac{M+L_3}{2} \norm{\vh_t}^2 + \epsilon + \sqrt{L_3\epsilon} \norm{\vh_t}.
\]
When $\norm{\vh_t} \leq \frac{1}{2}\sqrt{\frac{\epsilon}{L_3}}$ holds, combining it with $M=30L_3$ yields:
\[
\norm{ \nabla J(\theta_{t+1}) } \leq 6\epsilon.
\]

To analyze the Hessian, we use:
\[
\norm{ \nabla^2 J(\theta_{t+1}) - \nabla^2 J(\theta_t) } \leq L_3 \norm{\vh_t}.
\]
Hence, 
\[
\lambda_{\min}( \nabla^2 J(\theta_{t+1}) ) \geq \lambda_{\min}( \nabla^2 J(\theta_t) ) - L_3 \norm{\vh_t},
\]
and further,
\[
\lambda_{\min}( \nabla^2 J(\theta_{t+1}) ) \geq \lambda_{\min}( \mH_t ) - \sqrt{L_3\epsilon} - L_3 \norm{\vh_t}.
\]
From (\ref{eq:cubic-optimality-hessian}) we have:
\[
\lambda_{\min}( \mH_t ) \geq -\frac{M}{2} \norm{\vh_t},
\]
so we obtain:
\[
\lambda_{\min}( \nabla^2 J(\theta_{t+1}) ) \geq -\left( \frac{M}{2} + L_3 \right) \norm{\vh_t} - \sqrt{L_3\epsilon} \geq -9\sqrt{L_3\epsilon}.
\]
We conclude that $\theta_{t+1}$ satisfies the $\epsilon$-second-order stationarity conditions as defined in \cref{def:SOSP}:
\[
\norm{ \nabla J(\theta_{t+1}) } \leq 6\epsilon, \quad \lambda_{\min}( \nabla^2 J(\theta_{t+1}) )
\geq -9\sqrt{L_3\epsilon}.
\]
\end{proof}
\subsection{Variance Reduction Analysis}
\label{sec:variance-reduction}
In this subsection, our goal is to bound the deviation between the estimated gradients and Hessians and their true expectations. 
\paragraph{Estimator Definitions.}  
At each iteration $t$, we compute stochastic estimates based on sampled trajectories. Let $\tau_{t,i}$ denote the $i$-th trajectory sampled at time $t$. Recall that $X(t; \tau)$ denotes the cumulative log-policy up to time $t$. 
We define:
\begin{itemize}
  \item $\va_{t,i} := g(\theta_t \mid \tau_{t,i})$: single-sample gradient estimate at snapshot $\theta_t$;
  \item $\vb_{t,i} := H(\theta_t' \mid \tau_{t,i}) \vh_{t-1}$: single-sample correction term based on the previous step;
  \item $\mC_{t,i} := H(\theta_t \mid \tau_{t,i})$: single-sample estimate of the Hessian at $\theta_t$.
\end{itemize}
These estimators satisfy the following unbiasedness and bounded deviation properties:
\begin{itemize}
    \item $\EE[\va_t] = \nabla J_H(\theta_t)$, \quad and $\|\va_t - \EE \va_t\| \leq 2L_1$;
    \item $\EE[\vb_t] = \nabla J_H(\theta_t) - \nabla J_H(\theta_{t-1})$, \quad and $\|\vb_t - \EE \vb_t\| \leq 2L_2 \| \vh_{t-1} \|$;
    \item $\EE[\mC_t] = \nabla^2 J_H(\theta_t)$, \quad and $\|\mC_t - \EE \mC_t\| \leq 2L_2$.
\end{itemize}
The full estimators are then given by averages over i.i.d. samples of these quantities.

\begin{lemma}(Gradient Estimator Bound)
\label{lem:gradient-concentration}
Fix $P \in (0, 1)$. Suppose the sample sizes satisfy
\begin{align*}
b_g &\geq 2592L_1^2 [\log(3T/P)]^2\epsilon^{-2}, \\
B_g &\geq 2592L_2^2 [\log(3T/P)]^2S\epsilon^{-2}.    
\end{align*}
Then, with probability at least $1 - (S+1)P/3T$, we have
\[
\| \vg_t - \nabla J_H(\theta_t) \| \leq \frac{\epsilon}{2}.
\]
\end{lemma}
\begin{proof}
Let $t_s = t - t \bmod S$. By \cref{lem:gradient-concentration}, we have:
\begin{align}\label{ineq:a-azuma}
\|\vg_{t_s} - \nabla J_H(\theta_{t_s})\|
&= \frac{1}{b_g} \left\| \sum_{i=1}^{b_g} (\va_{t_s,i} - \EE \va_{t_s,i}) \right\| \notag\\
\leq \sqrt{ \frac{36 L_1^2 \log(3T/P) }{b_g} }
&\leq \frac{\epsilon}{\sqrt{72 \log(3T/P)}}.
\end{align}
Similarly, for each correction step $t_j \in (t_s, t]$:
\begin{align}\label{ineq:b-azuma}
&\| (\vg_{t_j} - \vg_{t_{j-1}}) - (\nabla J_H(\theta_{t_j}) - \nabla J_H(\theta_{t_{j-1}})) \| \notag\\
&= \frac{1}{b_{g,t_j}'} \left\| \sum_{i=1}^{b_{g,t_j}'} (\vb_{t_j,i} - \EE \vb_{t_j,i}) \right\|
\leq \sqrt{ \frac{36 L_2^2 \log(3T/P) }{B_g} } \notag\\
&\leq \frac{\epsilon}{\sqrt{72 \log(3T/P) S}}.
\end{align}
Aggregating over all terms and applying union bound:
\begin{align}\label{ineq:g-azuma}
&\|\vg_t - \nabla J_H(\theta_t)\| \notag\\
&\leq 3\sqrt{\log(3T/P) \left( \frac{\epsilon^2}{72 \log(3T/P)} + \frac{\epsilon^2 (t - t_s)}{72 \log(3T/P) S} \right)} \leq \frac{\epsilon}{2}.
\end{align}

Note that this bound holds uniformly over all $t \in \{0, \dots, T-1\}$ with probability at least $1 - 2P/3$ instead of $1-(S+1)P/3$, because the estimators (\ref{ineq:g-azuma}) at different time steps share a large amount of common randomness (\ref{ineq:a-azuma}), (\ref{ineq:b-azuma}).
\end{proof}

\begin{lemma}(Hessian Estimator Bound)
\label{lem:hessian-concentration}
Fix $P \in (0, 1)$. Suppose the sample size satisfies
\[
b_H \geq 144 L_2^2 L_3^{-1} \log(3Td/P) \epsilon^{-1}.
\]
Then with probability at least $1 - P/3T$, we have
\[
\| \mH_t - \nabla^2 J_H(\theta_t) \| \leq \frac{1}{2} \sqrt{L_3 \epsilon}.
\]
\end{lemma}    

\begin{proof}
We apply \cref{lem:hessian-concentration} to the deviation sum:
\begin{align*}
&\|\mH_t - \nabla^2 J_H(\theta_t)\| = \frac{1}{b_H} \left\| \sum_{i=1}^{b_H} (\mC_{t,i} - \EE \mC_{t,i}) \right\|\\
&\leq \sqrt{ \frac{36 L_2^2 \log(3Td/P)}{b_H} } \leq \frac{1}{2} \sqrt{L_3 \epsilon}.
\end{align*}    
\end{proof}

\begin{lemma}
\label{lem:combined-prob}
Assuming that the values of $b_g$, $B_g$, and $b_H$ satisfy the conditions specified in \cref{thm:main}, then, with probability at least $1 - P$, we have
\begin{align*}
    \| \vg_t - \nabla J_H(\theta_t) \| \leq \frac{\epsilon}{2}, \quad
    \| \mH_t - \nabla^2 J_H(\theta_t) \| \leq \frac{1}{2} \sqrt{L_3 \epsilon}
\end{align*}
for all time steps $t$.
\end{lemma}
\begin{proof}
Combining \cref{lem:gradient-concentration} and \cref{lem:hessian-concentration}, we conclude that \cref{lem:combined-prob} hold with probability no less than $1 - \frac{2P}{3} - T\cdot\frac{P}{3T} = 1 - P$.    
\end{proof}

\subsection{Truncation Analysis}
\label{sec:truncation}
In this subsection, we quantify the error introduced by using the truncated surrogate objective $J_H(\theta)$ in place of the full return objective $J(\theta)$ in the gradient and Hessian estimates.

\begin{lemma}
\label{lem:truncation-bias}
Let $\Gamma_1 := \frac{1}{\log(1/\gamma)}$, $\Gamma_2 := \frac{1}{1 - \gamma}$. If the truncation horizon $H$ satisfies
\begin{align*}
H \geq \max\Big\{ &
2\Gamma_1 \log\left(4 \Gamma_1 \Gamma_2 G_1 R_{\max} \epsilon^{-1}\right) + \Gamma_2, \\
&2\Gamma_1 \log\left(8 \Gamma_1 \Gamma_2 G_2 R_{\max} L_3^{-1/2} \epsilon^{-1/2}\right) + \Gamma_2, \\
&2\Gamma_1 \log\left(64 \Gamma_1^2 \Gamma_2 G_1^2 R_{\max} L_3^{-1/2} \epsilon^{-1/2}\right) + 6\Gamma_2
\Big\},
\end{align*}
then the following bounds hold:
\begin{align*}
\norm{ \nabla J_H(\theta) - \nabla J(\theta) } &\leq \frac{\epsilon}{2}, \\
\norm{ \nabla^2 J_H(\theta) - \nabla^2 J(\theta) } &\leq \frac{1}{2}\sqrt{L_3 \epsilon}.
\end{align*}
\end{lemma}

We first bound the gradient bias due to truncation:
\begin{align*}
&\norm{ \nabla J_H(\theta) - \nabla J(\theta) } \\
&= \left\| \EE \sum_{k=H}^\infty \gamma^k r_k \sum_{k'=0}^{k} \nabla_\theta \log \pi(a_{k'} | s_{k'}) \right\| \\
&\leq \sum_{k=H}^\infty \gamma^k R_{\max} (k+1) G_1 
= G_1 R_{\max} \sum_{k=H}^\infty \gamma^k (k + 1) \\
&= G_1 R_{\max} \gamma^H \sum_{k'=0}^\infty \gamma^{k'} (H + k' + 1) \\
&= G_1 R_{\max} \frac{\gamma^H}{1 - \gamma} \left( H + \frac{1}{1 - \gamma} \right).
\end{align*}

To ensure this is less than $\epsilon/2$, we rearrange:
\begin{align*}
2G_1 R_{\max} \Gamma_2\gamma^{-\Gamma_2}\epsilon^{-1} (H + \Gamma_2)
&\leq \gamma^{-(H + \Gamma_2)} \\
\iff
\Gamma_1 \log\left(2G_1 R_{\max} \Gamma_2\gamma^{-\Gamma_2}\epsilon^{-1} (H + \Gamma_2) \right)
&\leq H + \Gamma_2,
\end{align*}
It follows from \cref{lem:log-ineq} that the above inequality holds under the following conditions: (note that $\Gamma_1 \leq \Gamma_2$)
\[
H \geq 2\Gamma_1 \log\left( 4\Gamma_1 \Gamma_2 G_1 R_{\max} \epsilon^{-1} \right) + \Gamma_2.
\]

We now turn to the Hessian bias:
\begin{align*}
&\norm{ \nabla^2 J_H(\theta) - \nabla^2 J(\theta) }\\
&\leq \sum_{k = H}^\infty \gamma^k R_{\max} \left[ G_2(k+1) + G_1^2(k+1)^2 \right] \\
&= G_2 R_{\max} \sum_{k=H}^\infty \gamma^k (k + 1) + G_1^2 R_{\max} \sum_{k=H}^\infty \gamma^k (k + 1)^2 \\
&= G_2 R_{\max} \frac{\gamma^H}{1 - \gamma} \left( H + \frac{1}{1 - \gamma} \right) \\
&\quad+ G_1^2 R_{\max} \frac{\gamma^H}{1 - \gamma} \left( H^2 + \frac{2H - 1}{1 - \gamma} + \frac{2}{(1 - \gamma)^2} \right).
\end{align*}

We require both terms to be bounded by $\sqrt{L_3 \epsilon} / 4$.

\paragraph{Term 1}
We require
\[
G_2R_{\max}\Gamma_2\gamma^H(H+\Gamma_2)\leq\tfrac{1}{4} \sqrt{L_3 \epsilon}
\]
which is equivalent to:
\[
\Gamma_1 \log\left( \frac{4G_2 R_{\max} \Gamma_2}{L_3^{1/2} \epsilon^{1/2} \gamma^{\Gamma_2}} (H + \Gamma_2) \right)
\leq H + \Gamma_2,
\]
so it suffices to choose
\[
H \geq 2\Gamma_1 \log\left( 8\Gamma_1 \Gamma_2 G_2 R_{\max} L_3^{-1/2} \epsilon^{-1/2} \right) + \Gamma_2.
\]

\paragraph{Term 2}
We bound:
\[
G_1^2 R_{\max} \Gamma_2 \gamma^H (H + 2\Gamma_2)^2 \leq \tfrac{1}{4} \sqrt{L_3 \epsilon},
\]
which is equivalent to:
\[
2\Gamma_1 \log\left( \sqrt{ \tfrac{4 G_1^2 R_{\max} \Gamma_2}{L_3^{1/2} \epsilon^{1/2} \gamma^{2\Gamma_2}} } (H + 2\Gamma_2) \right)
\leq H + 2\Gamma_2.
\]
Thus, it suffices to choose:
\[
H \geq 2\Gamma_1 \log\left( 64\Gamma_1^2 \Gamma_2 G_1^2 R_{\max} L_3^{-1/2} \epsilon^{-1/2} \right) + 6\Gamma_2.
\]

By bounding the gradient and Hessian truncation bias separately and selecting $H$ according to the derived expressions, we conclude that Lemma~\ref{lem:truncation-bias} holds. In particular, it suffices to set $H = \mathcal{O}(\Gamma_2 \log(\Gamma_2 \epsilon^{-1}))$.

\section{Proof of \cref{cor:total-sample}}
\label{proof:cor-total-sample}
We aim to compute the total number of gradient and Hessian samples needed to ensure that the variance-reduced estimators satisfy the deviation bounds with high probability across all $T$ iterations.

\subsection{Hessian Sample Complexity.}
We set $S = \frac{L_1}{L_2} \sqrt{ \frac{L_3}{\epsilon} }$ to balance the sample complexity between gradient and Hessian terms. The total Hessian sample complexity is
\begin{align*}
&T \cdot b_H + \sum_{t \bmod S \neq 0} b_{g,t}' \\
&= \mathcal{O} \left( \log(3Td/P) \cdot \epsilon^{-2.5} \right)\\
&\quad + \mathcal{O} \left( \big( \log\left( 3T/P \right) \big)^2 \epsilon^{-0.5} \epsilon^{-2} \sum_{t=0}^{T-1} \norm{\vh_{t-1}}^2 \right) \\
&= \mathcal{O} \left( \log(3Td/P) \epsilon^{-2.5} + \big(\log(3T/P)\big)^2 \epsilon^{-2.5} \sum_{t=0}^{T-1} \norm{\vh_{t-1}}^2 \right).
\end{align*}

To bound the second term, we apply the Cauchy–Schwarz inequality:
\[
\sum_{t} \|\vh_t\|^3 \cdot\sum_{t} \|\vh_t\|^3 \cdot \sum_{t} 1 \geq \left(\sum_{t} \|\vh_t\|^2\right)^3
\]
Using the fact that $\sum_{t=0}^{T-1} \norm{\vh_{t}}^3 = \mathcal{O}(1), T = \mathcal{O}(\epsilon^{-1.5})$, we have:
\[
\sum_{t=0}^{T-1} \norm{\vh_{t-1}}^2 = \mathcal{O}(\epsilon^{-0.5}),
\]
which implies:
\[
\text{Hessian sample complexity} = \mathcal{O} \left( [\log(3T/P)]^2 \epsilon^{-3} \right).
\]

\subsection{Gradient Sample Complexity.}
The gradient sample complexity is simpler:
\begin{align*}
 b_g \cdot T/S &= \mathcal{O} \left( \left[ \log\left( 3T/P \right) \right]^2 \epsilon^{-2} \cdot \epsilon^{-1.5} \cdot \epsilon^{0.5} \right) \\
 &= [\log(3T/P)]^2 \epsilon^{-3}.    
\end{align*}

Substituting $T = \mathcal{O}(\epsilon^{-1.5})$ yields the final bound:
\[
\text{Gradient: } \mathcal{O} \left( \epsilon^{-3} [\log(1/\epsilon)]^2 \right), \quad
\text{Hessian: } \mathcal{O} \left( \epsilon^{-3} [\log(1/\epsilon)]^2 \right).
\]
\qed

\section{Proof of Boundedness for Log-Linear Policy Class in \cref{sec:experiments}}
\label{app:loglinear-regularity}

We verify that the commonly used log-linear policy class satisfies the boundedness assumptions in \cref{a-1,a-2,a-3} on the first three derivatives of $\log \pi_\theta(a \mid s)$.

\paragraph{Policy Definition.}
We consider a policy of the form
\[
\pi_\theta(a \mid s) = \frac{\exp(\theta^\top \phi(s, a))}{\sum_{a'} \exp(\theta^\top \phi(s, a'))},
\]
where $\phi(s, a) \in \mathbb{R}^d$ is a fixed feature vector. This log-linear family includes softmax policies over discrete actions as a special case.

\paragraph{Derivatives of the Log-Policy.}
The first three derivatives of the log-policy take the following forms:
\begin{align*}
&\nabla \log \pi_\theta(a \mid s)
= \phi(s, a) - \bar\phi(s), \\
&\nabla^2 \log \pi_\theta(a \mid s)\\
&= -\sum_{a'} \pi_\theta(a' \mid s) \left( \phi(s, a') - \bar\phi(s) \right)\left( \phi(s, a') - \bar\phi(s) \right)^\top, \\
&\nabla^3 \log \pi_\theta(a \mid s)[\vy] \\
&= -\sum_{a'} \pi_\theta(a' \mid s) \cdot \left( \left( \phi(s, a') - \bar\phi(s) \right)^\top \vy \right) \\
&\quad\cdot \left( \phi(s, a') - \bar\phi(s) \right)\left( \phi(s, a') - \bar\phi(s) \right)^\top,
\end{align*}
where $\bar\phi(s) := \sum_{a} \pi_\theta(a \mid s) \phi(s, a)$ is the expected feature vector under $\pi_\theta$.

\paragraph{Uniform Boundedness.}
Assume there exists a constant $C_\phi > 0$ such that $\| \phi(s, a) \| \leq C_\phi$ for all $(s, a)$. Then:
\begin{itemize}
  \item First-order bound:
  \[
  \| \nabla \log \pi_\theta(a \mid s) \| \leq \| \phi(s, a) \| + \| \bar\phi(s) \| \leq 2 C_\phi,
  \]
  so we may set $G_1 = 2 C_\phi$.
  
  \item Second-order bound: each term in the sum is of the form $\vv \vv^\top$ with $\| \vv \| \leq 2 C_\phi$, so we have:
  \[
  \| \nabla^2 \log \pi_\theta(a \mid s) \| \leq 4 C_\phi^2,
  \]
  so we may set $G_2 = 4 C_\phi^2$.

  \item Third-order directional derivative: for all $\| \vy \| = 1$,
  \[
  \| \nabla^3 \log \pi_\theta(a \mid s)[\vy] \| \leq 2 C_\phi \cdot (2 C_\phi)^2 \leq 8 C_\phi^3,
  \]
  so we may set $G_3 = 8 C_\phi^3$.
\end{itemize}

Therefore, log-linear policies satisfy the regularity assumptions \ref{a-1}, \ref{a-2} and \ref{a-3} with explicit constants determined by the feature norm bound $C_\phi$.

\end{document}